\documentclass[preprint]{article}
\usepackage{icml2026/icml2026}
\usepackage{microtype}
\usepackage{graphicx}
\usepackage{subfigure}
\usepackage{booktabs}
\usepackage{hyperref}
\usepackage{mystyle}
\usepackage{makecell}

\newcommand{\EquiTriangle}[4][]{%
  \path
    let \p1 = (#2),
        \n1 = {#3},
        \n2 = {0.5*sqrt(3)*#3}
    in
    coordinate (#4A) at (\x1,          \y1)          % first corner
    coordinate (#4B) at (\x1+\n1,      \y1)          % second corner
    coordinate (#4C) at (\x1+0.5*\n1,  \y1+\n2);     % top corner
  \draw[#1] (#4A) -- (#4B) -- (#4C) -- cycle;
}
\usepackage{titlesec}
\titlespacing*{\paragraph}{0pt}{0mm plus 0mm minus 0mm}{2mm plus 0mm minus 0mm}
\titlespacing*{\section}{0pt}{0mm plus 0mm minus 0mm}{0mm plus 0mm minus 0mm}
\titlespacing*{\subsection}{0pt}{0mm plus 0mm minus 0mm}{0mm plus 0mm minus 0mm}

\icmltitlerunning{Deep Neural Networks as Iterated Function Systems and a Generalization Bound}

\begin{document}

\twocolumn[
\icmltitle{Deep Neural Networks as Iterated Function Systems and a Generalization Bound}
\icmlsetsymbol{equal}{*}
\begin{icmlauthorlist}
\icmlauthor{Jonathan Vacher}{equal,yyy}
%\icmlauthor{}{sch}
\end{icmlauthorlist}
\icmlaffiliation{yyy}{Université Paris Cité, CNRS, MAP5, F-75006 Paris, France}
%\icmlaffiliation{comp}{Company Name, Location, Country}
%\icmlaffiliation{sch}{School of ZZZ, Institute of WWW, Location, Country}
\icmlcorrespondingauthor{Jonathan Vacher}{jonathan.vacher@u-paris.fr}
%\icmlcorrespondingauthor{Firstname2 Lastname2}{first2.last2@www.uk}
\icmlkeywords{Transformers, ResNet, Stochastic Iterated Function Systems, Random Dynamical Systems, Generalization, Generative Models}
\vskip 0.3in
]

\printAffiliationsAndNotice{}

\begin{abstract}
Deep neural networks (DNNs) achieve remarkable performance on a wide range of tasks, yet their mathematical analysis remains fragmented: stability and generalization are typically studied in disparate frameworks and on a case-by-case basis. Architecturally, DNNs rely on the recursive application of parametrized functions, a mechanism that can be unstable and difficult to train, making \emph{stability} a primary concern. Even when training succeeds, there are few rigorous results on how well such models \emph{generalize} beyond the observed data, especially in the generative setting. In this work, we leverage the theory of stochastic Iterated Function Systems (IFS) and show that two important deep architectures can be viewed as, or canonically associated with, place-dependent IFS. This connection allows us to import results from random dynamical systems to (i) establish the existence and uniqueness of invariant measures under suitable contractivity assumptions, and (ii) derive a Wasserstein generalization bound for generative modeling. The bound naturally leads to a new training objective that directly controls the collage-type approximation error between the data distribution and its image under the learned transfer operator. We illustrate the theory on a controlled 2D example and empirically evaluate the proposed objective on standard image datasets (MNIST, CelebA, CIFAR-10).
\end{abstract}

\section{Introduction}

Deep neural networks (DNNs) are now the default modeling tool for high–dimensional prediction and generation, yet their theoretical understanding remains fragmented. Many successful architectures, \eg{} ResNets, Transformers, and Mixture-of-Experts (MoE) layers, are built from the repeated composition of parametrized maps, sometimes with stochastic routing. This recursive structure naturally invites a dynamical-systems viewpoint, but most existing analyses either linearize locally, pass to continuous-time limits, or focus on optimization dynamics (\eg{} stochastic gradient descent) rather than on the action of the trained network itself. As a result, we still lack a unified framework that simultaneously (i) captures the depth-wise dynamics of modern architectures, (ii) yields stability and convergence guarantees, and (iii) provides task-relevant generalization bounds, in particular for generative models.

\paragraph{Random Dynamical Systems} Iterated function systems (IFSs) offer a compact language for contractive dynamical systems whose attractors are fractal sets or invariant measures. Deterministic IFSs~\cite{BarnsleyHutchinson1989} yield existence and uniqueness of attractors via the Banach fixed point theorem and the associated ``collage'' principle, while stochastic variants \eg{} independent IFS, place-dependent IFS, and random matrix products extend these guarantees to Markovian compositions of maps~\cite{BarnsleyEltonGeronimo1988,DiaconisFreedman1999,Stenflo2003}. At the level of measures, contraction in Wasserstein distance leads to invariant laws and quantitative control of approximation error. Early neural-network work already hinted at this connection~\cite{stark1991neural,bressloff1991learning}, but the IFS viewpoint has remained largely underexploited in the context of modern deep architectures and large-scale generative modeling. We note two exceptions : one in which IFSs are used to analyze the dynamic of stochastic gradient descent~\cite{camuto2021fractal}; a second in which ``collage'' is used as an objective~\cite{poli2022selfsimilarity} without leveraging the stochastic IFS theory. 

%Iterated function systems (IFS) provide a concise language for describing contractive dynamical systems whose attractors are fractal sets or invariant measures. Classical deterministic theory (Hutchinson and Barnsley) gives existence and uniqueness of attractors via the Banach fixed point theorem, while stochastic variants (random IFS, place-dependent IFS, and random matrix products) extend these guarantees to Markovian compositions of maps~\cite{BarnsleyEltonGeronimo1988,DiaconisFreedman1999,Stenflo2003}. Early neural network work already connected these ideas~\cite{stark1991neural,bressloff1991learning}, yet the perspective has been underused in modern deep learning. We found one recent work using the IFS framework to analyze the dynamic of stochastic gradient descent~\cite{camuto2021fractal}. Here, we analyze directly the action of the neural networks on their inputs.

\paragraph{Deep Learning and Generative Models} Modern deep generative models are dominated by deep architectures: diffusion and score-based models rely on iterated denoising networks~\cite{song2021scorebased}; GANs use deep discriminators and generators trained in adversarial games~\cite{goodfellow2014generative}; VAEs and normalizing flows combine deep encoders/decoders or invertible blocks~\cite{papamakarios2021normalizing}; and Transformers underpin autoregressive and masked-generation models~\cite{vaswani2017attention}. Despite impressive empirical performance, these families typically offer only partial guarantees: diffusion models trade sample quality for slow iterative sampling~\cite{karras2022elucidating}; GANs can be unstable and mode-collapse-prone~\cite{cobbinah2025stableGANs}; VAEs blur high frequencies~\cite{bredell2023explicitly}; normalizing flows restrict expressivity for tractable Jacobians~\cite{papamakarios2021normalizing}; and Transformers, even when their depth dynamics have been analyzed~\cite{geshkovski2023emergence,geshkovski2023mathematical}, come with no general guarantee that the empirical training distribution is close to an invariant attractor of the learned dynamics. More generally, standard architectural choices (residual connections, self-attention, normalization layers, drop-out etc.) are still largely guided by empirical observations rather than explicit stability criteria, which leads to ad-hoc design “recipes”. In contrast, several works show that explicit control of the Lipschitz constant or spectral norm is crucial for training stability~\cite{miyato2018spectral,delattre2023efficient}.

In this work, we revisit these questions through the lens of stochastic IFSs and random dynamical systems. At a qualitative level, this viewpoint turns the depth evolution of a network into a random dynamical system acting on measures, so that stability becomes contractivity of an associated transfer operator and generalization in generative modeling becomes the proximity between its invariant law and the data distribution. This perspective suggests a collage-type training principle: instead of matching a generic normal distribution to the empirical distribution using deep architecture (GANs, VAEs, normalizing flows, diffusion models/flow matching), one directly controls how far a single application of the learned operator moves the empirical distribution.

%\vspace{-4mm}
\paragraph{Contributions.}
Our contributions are threefold: (i) We provide a unified stochastic-IFS interpretation of several standard deep architectures by making explicit the branch maps and selector kernels of ResNets, Transformers, and Mixture-of-Experts layers, and we identify conditions under which their depth dynamics admit unique invariant measures and well-defined attractors. (ii) We derive a generalization bound for generative modeling in Wasserstein distance, showing that, under a contractivity assumption on the transfer operator, the distance between the model’s invariant law and the data distribution is controlled by an empirical “collage error” plus a purely statistical sampling term; this bound directly motivates a new training objective that penalizes the discrepancy between the empirical measure and its image under the learned operator. (iii) We instantiate this objective with spectrally-normalized MoE/I-IFS models, approximate the collage error with regularized optimal transport~\cite{peyre2019computational}, and show on both a controlled 2D example and on latent representations of MNIST, CelebA, and CIFAR-10 that minimizing the collage error yields meaningful attractors and promising qualitative results without diffusion steps or adversarial training\footnote{See online code: \url{https://github.com/xxx/xxx}}.

%\vspace{-4mm}
\paragraph{Notation.}
Throughout the manuscript, $(\Xx,d)$ denotes a complete separable metric space with Borel $\sigma$-algebra and $\Pp(\Xx)$ is the set of probability measures on $\Xx$ with second-order finite moments. The $K$-simplex is $\Delta^{K-1}$ and for any $p \in \Delta^{K-1}$, $\Cc(p)$ is the categorical distribution. Finally, we denote $d_H$ the Hausdorff distance and $\mathrm{W}_2$ the 2-Wasserstein distance. Definitions and proofs are given in supplementary Section~\ref{app:supp-def}.

\begin{figure}[t]
  \centering
  %---------------------- Panel (a): Fractal IFS ----------------------%
    \begin{tikzpicture}[scale=1, >=latex, font=\small]
        \draw[step=1.0,white,thin] (-4,-1.3) grid (4,2.7); 
    
        \coordinate (A) at (-3.8,0);
        \EquiTriangle[thick,fill=black!30]{A}{2cm}{T}
        \node at (-2.8,2.2){$T_0$};

        \coordinate (B) at (-1,0);
        % (i) f1(x) = x/2
        % (ii) f2(x) = (x + (1,0))/2 = 0.5*A + (0.5,0)
        % (iii) f3(x) = (x + (0.5, sqrt(3)/2))/2
        %             = 0.5*A + (0.25, sqrt(3)/4)
        \coordinate (B1) at ($(B)$);
        \coordinate (B2) at ($(B) + (1,0)$);
        \coordinate (B3) at ($(B) + (0.5,{sqrt(3)/2})$);

        \EquiTriangle[thick,fill=black!30]{B1}{1cm}{T}
        \EquiTriangle[thick,fill=black!30]{B2}{1cm}{T}
        \EquiTriangle[thick,fill=black!30]{B3}{1cm}{T}
        \node at (0,2.2){$T_1 = \Hh(T_0)$};

        \coordinate (C1) at (1.8,0);
        \coordinate (C11) at ($(C1)$);
        \coordinate (C12) at ($(C1) + 0.5*(1,0)$);
        \coordinate (C13) at ($(C1) + 0.5*(0.5,{sqrt(3)/2})$);
        \EquiTriangle[thick,fill=black!30]{C11}{5mm}{T}
        \EquiTriangle[thick,fill=black!30]{C12}{5mm}{T}
        \EquiTriangle[thick,fill=black!30]{C13}{5mm}{T}

        \coordinate (C2) at (2.8,0);
        \coordinate (C21) at ($(C2)$);
        \coordinate (C22) at ($(C2) + 0.5*(1,0)$);
        \coordinate (C23) at ($(C2) + 0.5*(0.5,{sqrt(3)/2})$);
        \EquiTriangle[thick,fill=black!30]{C21}{5mm}{T}
        \EquiTriangle[thick,fill=black!30]{C22}{5mm}{T}
        \EquiTriangle[thick,fill=black!30]{C23}{5mm}{T}

        \coordinate (C3) at (2.3,0.866);
        \coordinate (C31) at ($(C3)$);
        \coordinate (C32) at ($(C3) + 0.5*(1,0)$);
        \coordinate (C33) at ($(C3) + 0.5*(0.5,{sqrt(3)/2})$);
        \EquiTriangle[thick,fill=black!30]{C31}{5mm}{T}
        \EquiTriangle[thick,fill=black!30]{C32}{5mm}{T}
        \EquiTriangle[thick,fill=black!30]{C33}{5mm}{T}
        \node at (2.8,2.2){$T_2 = \Hh(T_1)$};

        \node[align=left, text width=5cm] at (-1,-0.5) {Hutchinson iterates : $T_{n+1} = \Hh(T_n)$};
        \node[align=left, text width=5cm] at (-1,-1.0) {Attractor set : Sierpiński triangle};

        \draw (-2,-1.3)--(2,-1.3);
    \end{tikzpicture}
    %--------- Panel (b): Fractal image encoding -----------------%

    \begin{tikzpicture}[scale=1, >=latex, font=\small]
        \draw[step=1.0,white,thin] (-4,-6) grid (4,3.5); 

        \node[right,align=left, text width=2.5cm] at (-2.7,3.25){Patches of image $I$};
        
        \node[right, align=left, text width=1cm] at (-3.2,2.8){source};
        \node[rotate=30] at (-2.7,1.6){\includegraphics[width=5mm]{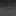}};
        \node[rotate=-15] at (-3.4,1.6){\includegraphics[width=5mm]{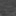}};
        \node[rotate=10] at (-2.9,2.2){\includegraphics[width=5mm]{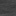}};
        \node[rotate=-20] at (-2.4,1.0){\includegraphics[width=5mm]{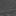}};
        \node[rotate=5] at (-3.1,1.0){\includegraphics[width=5mm]{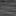}};
        \node[rotate=-5] at (-2.2,2.2){\includegraphics[width=5mm]{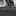}};

        \node[right,align=left, text width=1cm] at (-0.9,2.8){target};
        \node[rotate=-30] at (-0.3,1.6){\includegraphics[width=4mm]{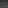}};
        \node[rotate=10] at (-1.0,1.6){\includegraphics[width=4mm]{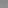}};
        \node[rotate=15] at (-0.9,2.2){\includegraphics[width=4mm]{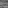}};
        \node[rotate=-10] at (0.0,1.0){\includegraphics[width=4mm]{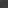}};
        \node[rotate=-5] at (-0.9,1.0){\includegraphics[width=4mm]{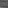}};
        \node[rotate=-10] at (-0.2,2.2){\includegraphics[width=4mm]{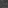}};

        \node[right,align=left, text width=3.4cm]  at (0.5,2.8){$\tau_{i\to j}$: translation of source patch $i$ to target $j$.};
        \node[right,align=left, text width=3.4cm]  at (0.5,1.4){$w_k$: simple affine function (rotation, flip, $\nicefrac{1}{2}$-patch scaling, gray level scaling/translation, etc.).};

        \node[right,align=left, text width=8cm]  at (-3.95,0.3){For all source patch $i$, find $(j_i,k_i)$ such that $\Hh(I)\simeq I$.};
        \node[right,align=left, text width=8cm]  at (-3.95,-0.2){Fixed point of $\Hh$ for IFS $\Ww=\{\tau_{i\to j_i} \circ w_{k_i}\}_i$ is close to $I$.};

        \node at (0,-3.5){\includegraphics[width=8cm]{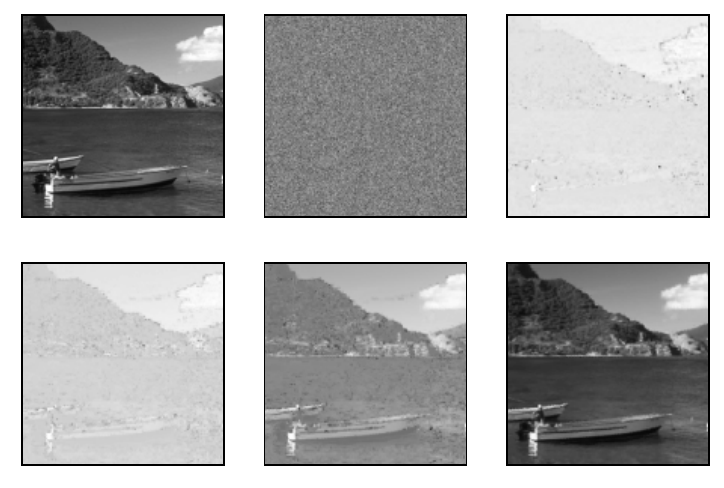}};
        \node at (-2.6,-0.7){$I$};
        \node at (0,-0.7){$I_0$};
        \node at (2.8,-0.7){$I_1 = \Hh(I_0)$};
        \node at (-2.5,-3.5){$I_2 = \Hh(I_1)$};
        \node at (0,-3.5){$I_4 = \Hh(I_3)$};
        \node at (2.8,-3.5){$I_{50} = \Hh(I_{49})$};

    \end{tikzpicture}

    \vspace{-3mm}
    \caption{Schematic illustration of fractal generation with IFSs. Top: classical fractal attractor construction. Bottom: fractal encoding of an image, where contractive maps learned from patch correspondences define an IFS whose invariant set approximates the target image.}
    \label{fig:ifs-fractal-both}
\end{figure}
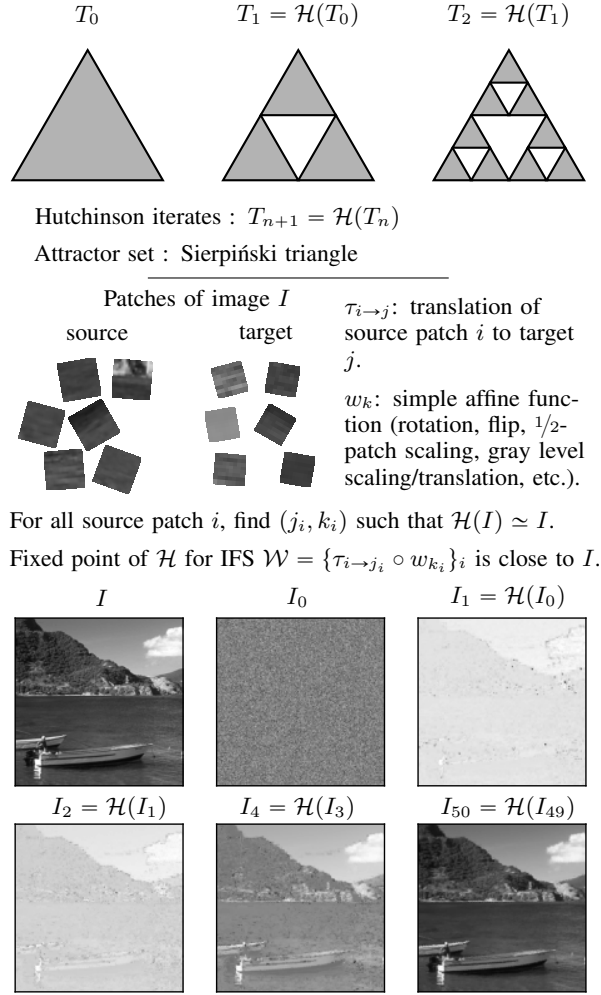

\subsection*{Preliminary Result}

A linear operator $T: \Xx \to \Xx$ is $c$-contractive when its Lipschitz constant $c$ belong to $[0,1[$. In this case, the Banach fixed point theorem applies (see supplementary Section~\ref{app:supp-def}). A key corollary states that if one has a vector that is $\varepsilon$-invariant by $T$ then this vector is close enough to the fixed point of $T$. This formalizes as it follows.

\begin{cor}[``Collage'' in Metric Space]
    \label{cor:metric-collage}
    Let $T : \mathcal X \to \mathcal X$ a contraction with constant $c \in [0,1)$. Let $\tilde{x} \in \mathcal X$ be the (unique) fixed point of $T$, \ie{} $T\tilde{x} = \tilde{x}$.
    
    Then, for any $y \in \mathcal X$, we have the following implication
    \eql{
        d\bigl(Ty,y\bigr) \leq \varepsilon  \quad\Longrightarrow\quad   d(\tilde{x},y) \le \frac{\varepsilon}{1-c}.
        \label{eq:collage-ineq}
    }
\end{cor}

This result will be used along this manuscript for different metric spaces. It is the key result of fractal image compression~\cite{barnsley1996fractal}.

\section{Iterated Function Systems}
\label{sec:ifs}

\subsection{Deterministic IFS}
\label{sec:deter-ifs}

\begin{defn}[IFS]
    Let $\Ii=\{1,\dots,m\}$ be a finite index set and $\Ww=\{w_\xi:\mathcal{X}\to \mathcal{X}\}_{\xi\in\Ii}$ a finite family of functions. We call $\Ww$ an iterated function system (IFS).
    We say that an IFS is $c$-contractive when all its functions are at most $c$-contractive.
\end{defn}

Equip $\mathcal{K}(\mathcal{X})$, the set of nonempty compact subsets of $\mathcal{X}$, with the Hausdorff distance $d_H$. The space $(\Kk(\Xx), d_H)$ is a complete metric space~\cite{henrikson1999completeness}. Define the Hutchinson operator
\eq{
    \Hh(K) = \bigcup_{\xi\in\Ii} w_\xi(K).
}

\begin{thm}[Original Collage Theorem]
    \label{thm:collage}
    Let $\Ww$ be a $c$-contractive IFS. Then $\Hh$ is $c$-contractive on $(\mathcal{K}(\mathcal{X}),d_H)$ and Corollary~\ref{cor:metric-collage} applies.
\end{thm}

Theorem~\ref{thm:collage} is the original ``collage'' theorem. It states that if a compact set $K$ is $\varepsilon$-invariant ($d_H(K,\Hh(K))\le \varepsilon$), then it is close enough to the fixed point $\tilde{K}$ of the Hutchinson operator ($d_H(K,\tilde{K})\le \varepsilon/(1-c)$). The fixed point can be recovered using the recursive iteration of the operator. Classical applications of this result include line contractions, the Sierpiński gasket (Figure~\ref{fig:ifs-fractal-both}-top), and fractal image compression where the maps encode coarse-to-fine patch mappings~\cite{barnsley1996fractal} (Figure~\ref{fig:ifs-fractal-both}-bottom). More than that, IFSs are widely studied in the field of random dynamical systems.

\subsection{Stochastic IFSs}
\label{sec:stochastic-ifs}

Stochastic IFSs generalize the deterministic application of an operator depending on a family of functions by defining a recursion with a Markovian function selection. The most general type of stochastic IFSs is when the probability of choosing one function depends on the previous iterate. This defines place-dependent IFSs. A simpler case of stochastic IFS is when the probability of choosing one function does not depend on the previous iterate. This defines independent IFSs. 
\begin{defn}[Stochastic IFS]
    \label{def:place-ifs}
    Let $\Ww=\{w_\xi: \Xx \longrightarrow \Xx \}_{\xi\in\Ii}$ be an IFS and $p: \Xx \to \Delta^{|\Ii|-1}$ a selector kernel. Then $(\Ww,p)$ is a \emph{place-dependent IFS} (P-IFS) iff its Markov recursion is
    \eq{
        X_{t+1}=w_{\Xi_t}(X_t), \qquad \Xi_t \mid X_t=x_t \sim \Cc(p(x_t)).
    }
    When the selector kernel $p$ does not depend on $x$ \ie{} for all $x \in \Xx$, 
    \eq{
        p(x) = q \in \Delta^{|\Ii|-1},
    }
    the pair $(\Ww,q)$ is called an Independent IFS (I-IFS).
\end{defn}

Both I- and P-IFSs induce a Markov operator on bounded functions and an adjoint operator on measures. These operators allow to make the connection with the Hutchinson operator defined in Section~\ref{sec:deter-ifs}.

\begin{defn}[Markov operator, transfer operator and invariant measure]
    \label{def:markov-op}
    For $f\in C_b(\mathcal{X})$ the Markov operator is
    \eqal{
        (T f)(x)
          &= \mathbb{E}_{\Xi|X=x\sim \Cc(p(x))}[f(w_\Xi(x))]\\
          &= \sum_{\xi\in\Ii} p_\xi(x)\, f(w_\xi(x)).
    }
    Its adjoint, the transfer operator is $T^\ast$ acting on $\mathcal{P}(\mathcal{X})$, satisfies
    \eqal{
        T^\ast\mu
          &= \mathbb{E}_{X\sim\mu,\, \Xi|X\sim \Cc(p(X))}
            \big[\delta_{w_\Xi(X)}\big] \\
          &= \sum_{\xi\in\Ii} (w_\xi)\# ( p_\xi \cdot \mu ),
        \label{eq:state-dependent-operator}
    }
    where $(w_\xi)\#\mu$ is the pushforward of $\mu$ by $w_\xi$ and $p_\xi\cdot\mu$ has
    density $p_\xi$ against $\mu$. An \emph{invariant measure} satisfies
    $T^\ast\tilde\mu=\tilde\mu$.
\end{defn}

Intuitively $T^* \mu$ represents the distribution of the IFS output after one step when the input is drawn from $\mu$.

An important question related to stochastic IFSs is whether their transfer operator admit an invariant measure. Several theorems exist for I-IFS and P-IFS~\cite{BarnsleyEltonGeronimo1988,stenflo2002uniqueness,Stenflo2003}. The hypotheses of these theorems are generally based on the contractive nature of the functions as presented in the theorem below.

\begin{thm}[Invariant measure for P-IFS]
    \label{thm:invariant-measure-pifs}
    Let $(\Ww,p)$ be a P-IFS whose functions $(w_\xi)$ are Lipschitz with constants $(c_\xi)$.
    If $\sup_x \sum_{\xi\in\Ii} p_\xi(x) c_\xi^2 < 1$ (average contraction), then there exists a unique invariant measure $\tilde{\mu}$ for $T^\ast$, and the associated Markov recursion converges in law to $\tilde{\mu}$
\end{thm}

As mentioned earlier, the invariant measure of the transfer operator can be related to the fixed point of the Hutchinson operator as being its support.

\begin{prop}[Support of invariant measure and Hutchinson fixed-point]
    \label{prop:support-invariant-measure-hutchinson}
    Let $(\Ww,q)$ be a $c$-contractive I-IFS with all $q_\xi>0$. Let $\tilde{\mu}$ be the (unique) invariant measure of the transfer operator $T^\ast$ and let $\tilde{K} $ be the fixed point of the Hutchinson set operator $\Hh$. Then, $\tilde{K} = \operatorname{supp}(\tilde{\mu})$.
\end{prop}

Finally, in the context of measures we also have collage theorems. For I-IFS, contraction of the branch functions yields contraction of the transfer operator in Wasserstein distance~\cite{DiaconisFreedman1999,villani2008optimal}.

\begin{thm}[Collage Theorem on measure]
    \label{thm:stochastic-collage}
    Let $(\Ww,q)$ be a $c$-contractive I-IFS. Then, its transfer operator $T^\ast$ is contractive on $(\Pp(\Xx),\mathrm{W}_2)$ and Corollary~\ref{cor:metric-collage} applies.
\end{thm}

For general P-IFS, contractivity of each branch does not suffice to guarantee Wasserstein contraction of the operator. A Lipschitz condition on the selector kernel itself is required~\cite{meyn2012markov,villani2008optimal}. 

\paragraph{Consequence for generative modeling} This measure collage theorem has an important consequence for generative modeling. Given an empirical measure $\mu$, it means that if we are able to find a contractive transfer operator $T^\ast$ (or its corresponding stochastic IFS) such that $\mathrm{W}_2(T^\ast \mu, \mu) \leq \varepsilon$ then the attractor (or fixed-point) $\tilde{\mu}$ of $T^\ast$ will be close enough to $\mu$. We derive a bound on the generalization error in Section~\ref{sec:generalization}.

The deterministic Hutchinson operator, I-IFS and P-IFS are all versions of the same thing : the first applies all branches in parallel, the second averages over a fixed categorical draw, and the third selects branches in a state-dependent manner. Average contraction recovers the Banach fixed point intuition in each case. In the following section, we formalize how ResNet and Transformer architectures together with MoE fit in the IFS framework.

\section{Deep Neural Networks as IFS}

Each subsection below states a formal definition of the considered architecture, a formal definition of the function set and selector kernel pair $(\Ww,p)$ and a proposition establishing that it is a well-defined IFS. In this section, we assume that $\Xx=\RR^d$.

\subsection{ResNet as (degenerate) P-IFS}
\label{sec:resnet-pifs}

First, we consider the case of ResNet with ReLU activation function composed of the same residual block.

\begin{defn}[ReLU residual block]
    \label{def:residual-block}
    Let $m\in\NN$ be the width of a hidden layer, and let
    $A\in\RR^{m\times d}$, $b\in\RR^m$, $B\in\RR^{d\times m}$, $c\in\RR^d$ be parameters.
    With ReLU activation $\sigma(z)=\max\{z,0\}$ applied coordinate-wise, the residual
    block is the function $F_\theta:\Xx\to\Xx$ defined for all $x \in \Xx$ by
    \eq{
        F_\theta(x) = x + B\,\sigma(Ax + b) + c
    }
    where $\theta=(A,B,b,c)$.
\end{defn}

Iterations across depth are denoted $x^{(t+1)}=F_\theta(x^{(t)})$. Previous works have established this recursion as the Euler discretization of a so-called Neural ODE~\cite{chen2018neural} which allows to study theoretically the limit of infinitely deep neural networks using dynamical system theory. Here we keep the iterations discrete in depth. It is possible to consider a most general case where the parameter $\theta$ depends on depth but we prefer to keep it constant for simplicity. Indeed, considering the time dependence of $\theta$ only drastically increases the set of possible functions which stays finite in finite discrete time. In the following definition, given a ResNet we define an associated set of functions and a selector kernel.

\begin{defn}[Residual branch family and selector kernel]
    \label{def:residual-branch}
    Let $\Ii=\{0,1\}^m$ index activation patterns of the hidden layer.
    For $\xi=(\xi_1,\dots,\xi_m)\in\Ii$ define the diagonal matrix
    $D(\xi)=\mathrm{diag}(\xi)$ and the branch map $w_\xi^\theta:\Xx\to\Xx$ by
    \eq{
        w_\xi^\theta(x) = x + B\,D(\xi)(Ax + b) + c.
    }
    Set $\Ww_{\mathrm{ResNet}}=\{w_\xi^\theta\}_{\xi\in\Ii}$.
    For $x\in\Xx$, let $\xi(x)\in\Ii$ be the indicator of positive pre-activations, for all $j \in \{1,\dots,m\}$
    \eq{
        \xi_j(x)=\mathbf{1}\{(Ax+b)_j>0\}. %\mathds{P} \PP \mathds{R} \RR
    }
    We define the (degenerate) selector $p_{\mathrm{ResNet}}
    : x\in \Xx \mapsto p(x) \in \Delta^{|\Ii|-1}$ by
    \eq{
        p_{\xi}(x)= \delta_{\xi(x)}(\xi) =
          \begin{cases}
            1, & \text{if }\xi=\xi(x),\\[2pt]
            0, & \text{otherwise.}
          \end{cases}
    }
\end{defn}

\begin{prop}[ReLU residual block as place-dependent IFS]
    \label{prop:residual-ifs}
    The pair $(\Ww_{\mathrm{ResNet}},p_{\mathrm{ResNet}})$ defines a P-IFS on $\Xx$. The deterministic residual map $F_\theta$ is the following conditional expectation
    \eq{
        F_\theta(x)
          = \mathbb{E}_{\Xi | X=x \sim \Cc(p_{\mathrm{ResNet}}(x))}[w_\Xi^\theta(x)].
    }
\end{prop}

The above construction is not limited to ReLU activation, see supplementary Section~\ref{app:softplus-ifs} for a construction with softplus activation. Given an activation function, it is in fact often possible to build an IFS and there are often multiple possible constructions.

The selector kernel $p_{\mathrm{ResNet}}$ depends on $x$ only through the sign pattern of $Ax+b$, that is, through the ReLU activation pattern at the first linear layer.  Each pattern $\xi$ selects a subset of the columns of the subsequent linear map $B$ and therefore determines a specific affine map $w_\xi$ at each layer.  In fact, the selection is deterministic, so that a ReLU residual block corresponds to a degenerate (Dirac) selector kernel.  This activation-pattern viewpoint has been used by~\citet{Mohan2020Robust} to show that a bias-free ReLU network is a positively homogeneous linear function of order~1. However, no further dynamical analysis was developed.

Within the present IFS framework we can go further.  A ReLU residual block is a piecewise-affine P-IFS whose stability is governed by the average contraction Theorem~\ref{thm:invariant-measure-pifs}. Hence, imposing average contractivity guarantees stability at infinite depth. When such networks are trained for denoising, clean images are approximate fixed points; hence the empirical distribution of the training set is approximately invariant under the associated transfer operator.  Since each branch map is affine (or linear in the bias-free case), the asymptotic behavior of the network as depth tends to infinity can be analyzed using the classical theory of products of random matrices~\cite{furstenberg1960products,bougerol1985products} and affine IFS~\cite{BarnsleyEltonGeronimo1988}. In the linear (positively homogeneous) case the attractor of the induced dynamics is constrained to be a single point, a cone, or a finite union of cones, in agreement with the observations of~\citet{Mohan2020Robust}.  In the affine case the attractor may be a nontrivial fractal set, which is fully compatible with the empirically observed tendency of biased networks to encode the scale of the noise level present during training.

\subsection{Transformers as P-IFS}
\label{sec:tf-pifs}

Second, we consider the case of the Transformer architecture. Transformers depth dynamics have been recently studied and our interpretation corroborates previous findings about convergence towards degenerate measures~\cite{geshkovski2023emergence,geshkovski2023mathematical}. Yet, our approach slightly contrasts as it is valid on the context space $\Xx^n$, not the token space $\Xx$. The case of Transformers is also more interesting than the ResNet as the selector kernel is not degenerated and is explicitly related to the attention weights.

\begin{defn}[Transformer block]
    \label{def:transformer-block}
    Let $n$ be the context length and $\bx=(x_1,\dots,x_n)\in\Xx^n$ a context.
    Let $\ell_{\mathrm{norm}}$ denote layer normalization.
    Queries, keys and values are
    \eq{
        q_i( \bx )=W_q x_i,\; k_j( \bx )=W_k x_j,\; v_j( \bx )=W_v x_j,
    }
    where $(W_q,W_k,W_v)$ are matrices. Attention weights are
    \eq{
        \alpha_{ij}(\bx)
          = \frac{\exp(q_i( \bx)^\top k_j( \bx))}
                 {\sum_{\ell=1}^n \exp(q_i( \bx)^\top k_\ell( \bx))}.
    }
    The attention map is
    \eq{
        f_\theta^{\mathrm{SA}}(\bx)_i
          = \sum_{j=1}^n \alpha_{ij}(\bx) v_j(\bx) ,
    }
    and a two-layer MLP is
    \eq{
        g_\theta( \bx)
          = W_2\,\sigma(W_1  \bx + b_1) + b_2,
    }
    with activation $\sigma$ (\eg{} $\mathrm{ReLU}$ $\mathrm{GELU}$ or $\tanh$) applied coordinate-wise. The residual transformer block is the function $G_\theta: \Xx^n \to \Xx^n$ defined by for all $\bx \in \Xx^n$ by
    \begin{equation}
        G_\theta(\bx)
          = \bx + f_\theta^{\mathrm{SA}}(\ell_{\mathrm{norm}}(\bx))
              + g_\theta(\ell_{\mathrm{norm}}(\bx)).
        \label{eq:transformer-block}
    \end{equation}
\end{defn}

In practice, there are additional affine embedding and projector matrices / vectors $(W_{\mathrm{in}},b_{\mathrm{in}})$ and $(W_{\mathrm{out}}, b_{\mathrm{out}})$, though we omit these as they are factoring other parameters. Similarly, we do not define Multi-Head attention as it only consists in concatenating multiple single head and projecting back to the desired output space.

Again, iterations across depth on the context space $\Xx^n$ are denoted $\bx^{(t+1)}=G_\theta(\bx^{(t)})$, we do not need to consider the continuous limit as it usually done to analyze the behavior at infinite depth~\cite{geshkovski2023mathematical}. We can now define the set of functions and the selector kernel to establish the link with IFS.

\begin{defn}[Transformer branch family and selector kernel]
    \label{def:transform-branch}
    Let $\Ii=\{1,\dots,n\}^n$.
    For $\xi=(\xi_1,\dots,\xi_n)\in\Ii$ define
    \eqal{
        \big(w^\theta_\xi(\bx)\big)_i
          = x_i &+ v_{\xi_i}(\ell_{\mathrm{norm}}(\bx))\\
                &+ g_\theta(\ell_{\mathrm{norm}}(\bx))_i ,
        \qquad i\in\{1,\dots,n\},
    }
    and set $\Ww_{\mathrm{TF}}=\{w_\xi^\theta\}_{\xi\in\Ii}$.
    We define the selector kernel $p_{\mathrm{TF}}
    : \bx \in \Xx^n \mapsto p(\bx) \in \Delta^{|\Ii|-1}$ by
    \eq{
        p_{\xi}(\bx)
          = \prod_{i=1}^n \alpha_{i,\xi_i}(\bx),
    }
    with attention weights $\alpha_{i,j}(\bx)$ from
    Definition~\ref{def:transformer-block}.
\end{defn}

%\paragraph{State space.} As for the transformer block $G_\theta$, we consider dynamics that act on the \emph{context space} $\Xx^n$ and not on the token embedding space $\Xx$. Each map $w_\xi^\theta$ is a function $\Xx^n \to \Xx^n$, and the selector kernel $p_{\mathrm{TF}}$ is defined on $\Xx^n$.

The following proposition establishes definition~\ref{def:transform-branch} as a well-defined P-IFS.

\begin{prop}[Transformer as place-dependent random IFS]
    \label{prop:transformerIFS}
    The pair $(\Ww_{\mathrm{TF}}, p_{\mathrm{TF}})$ defines a
    P-IFS on $\Xx^n$. The deterministic Transformer block
    is the following conditional expectation
    \eq{
        G_\theta(\bx)
          = \mathbb{E}_{\Xi | \bX=\bx\sim \Cc(p_{\mathrm{TF}}(\bx))}
            \big[w^\theta_\Xi(\bx)\big].
    }
\end{prop}

Basically, this proposition says that a Transformer block can be viewed as randomly picking an attention pattern $\Xi$ according to the softmax weights and applying the corresponding update. In expectation this recovers the usual transformer output.

Contrary to the ResNet, the Transformer is not a degenerate IFS. In fact, the transformer block $G_\theta$ is only canonically associated to a distinct P-IFS as its conditional expectation. The associated P-IFS has a selector kernel defined by attention weights and non-linear branch functions $w^\theta_\xi$. The two-layer MLP $g_\theta$ is shared across all functions and only the linear part is selected ($v_{\xi_i}$). When the MLP is removed ($g_\theta=0$), the family $\{w^\theta_\xi\}$ reduces to an affine IFS and its asymptotic behavior can be characterized by the classical theory of products of random matrices~\citep{bougerol1985products} and affine iterated function systems~\citep{BarnsleyEltonGeronimo1988}. In this linear or affine setting, the attractor of the associated P-IFS is constrained to be a single point, a cone, a finite union of cones, or a fractal affine set. 

\begin{figure*}%[H]
    \centering
    \begin{tikzpicture}
        \draw[step=1.0,white,thin] (-8.5,-2.5) grid (8.5,2.7); 
        \node at (-5.5,0){\includegraphics[width=0.3\linewidth]{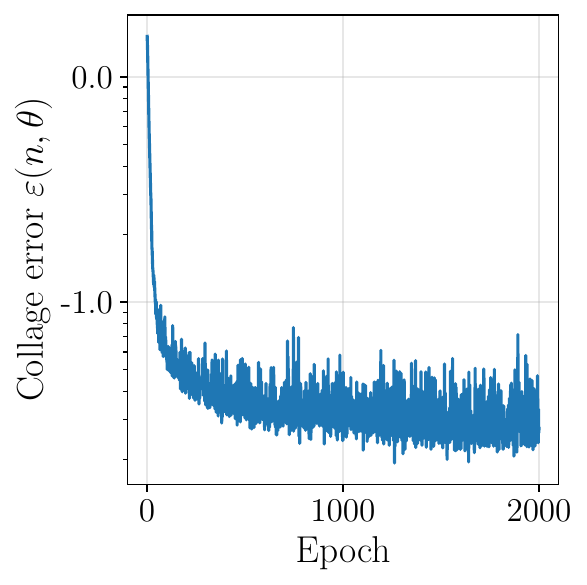}};
        \node at (0,0.15){\includegraphics[width=0.282\linewidth]{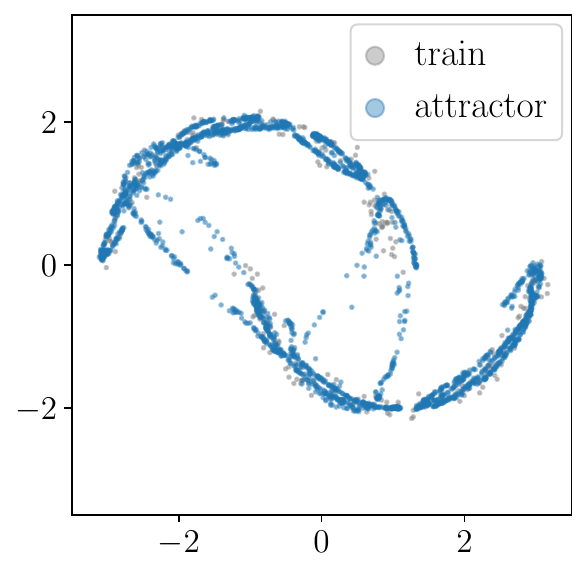}};
        \node at (5.5,0){\includegraphics[width=0.3\linewidth]{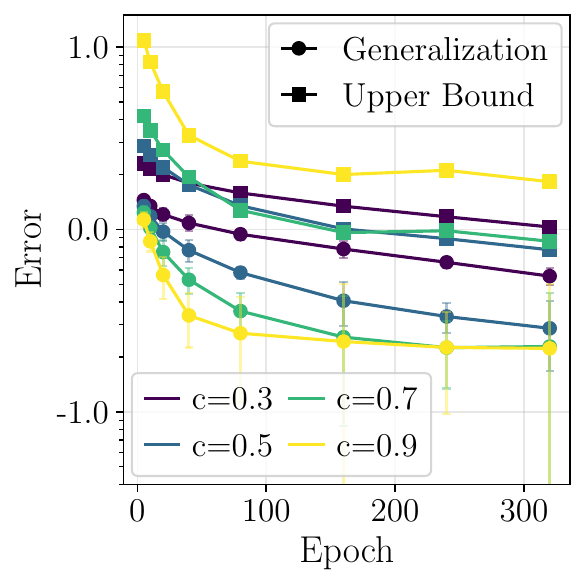}};
    \end{tikzpicture}    
    \vspace{-5mm}
    \caption{Left: Training collage error $\varepsilon(n,\theta)$ over epochs. Center: Training data and samples of the attractor measure of the trained MoE IFS. Right: Estimated generalization error and its bound (Theorem~\ref{thm:generalization-bound}) for different contraction constants.}
    \label{fig:two-moons-ifs}
\end{figure*}

\paragraph{Practical precedents of IFS-like sampling.} Stochastic or hard routing (\ie{} branch sampling) is already used to approximate dense attention: mixture-of-experts with sampled top--$k$ gating~\cite{shazeer2017moe,lepikhin2021gshard,fedus2022switch}, content-based sparse routing of keys~\cite{child2019sparse,roy2021routing}, hash-based neighbor sampling \cite{kitaev2020reformer}, and randomized feature estimators of softmax attention~\cite{choromanski2021performer}. These variants replace, conditionally on the input, the deterministic convex combination by sampled maps, matching the IFS formalism above.

Other deep learning architectures can be viewed as IFS and it might be interesting to complete an exhaustive classification. Though, now we would like to highlight that MoE can be canonically associated to an IFS.

\subsection{Mixture-of-Experts as P-IFS}

To complete we consider the case of MoE which have a very generic formulation. Using the right notations from the beginning will simplify the analysis.

\begin{defn}[Mixture-of-Experts]
    \label{def:moe-layer}
    Let $\Ii=\{1,\dots,K\}$ be an index set, and $\Theta$ a parameter space. Let $\theta \in \Theta$  and $\Ww_{\textrm{MoE}}=\{w_\xi^\theta\}_{\xi\in\Ii}$ be expert functions on $\Xx$. Let $p_{\mathrm{MoE}}=p^\theta:\Xx\to\Delta^{K-1}$ be a gating network. For all $ x\in\Xx$, the dense MoE function is
    \eq{
        H_\theta(x)=\sum_{\xi\in\Ii} p_\xi^\theta(x)\, w_\xi^\theta(x).
    }
\end{defn}

At this stage, it is clear that we have directly define the expectation of some P-IFS. This is made explicit in the following proposition.

\begin{prop}[MoE as place-dependent IFS]
    \label{prop:moe-pifs}
    The pair $(\Ww_{\mathrm{MoE}},p_{\mathrm{MoE}})$ defines a P-IFS on $\Xx$ with recursion
    \eq{
        x_{t+1}=w_{\Xi_t}^\theta(x_t), \qquad \Xi_t\mid X=x_t \sim \Cc(p^{\theta}(x_t)).
    }
    Moreover, the dense MoE output is the conditional expectation
    \eq{
        H_\theta(x)
          = \mathbb{E}_{\Xi | X=x \sim \Cc(p^{\theta}(x))}[w_\Xi^\theta(x)].
    }
\end{prop}

We can cascade MoE with different parameters at each layer as defined in the supplementary Section~\ref{app:deep-moe}. 

Supplementary Table~\ref{tab:ifs-architectures} summarize the IFS associated to ReLU-ResNet, Transformers and MoE.

\begin{figure*}%[H]
    \centering
    \begin{tikzpicture}
        \draw[step=1.0,white,thin] (-8.5,-2.5) grid (8.5,2.7); 
        \node at (-5.5,0){\includegraphics[width=0.282\linewidth]{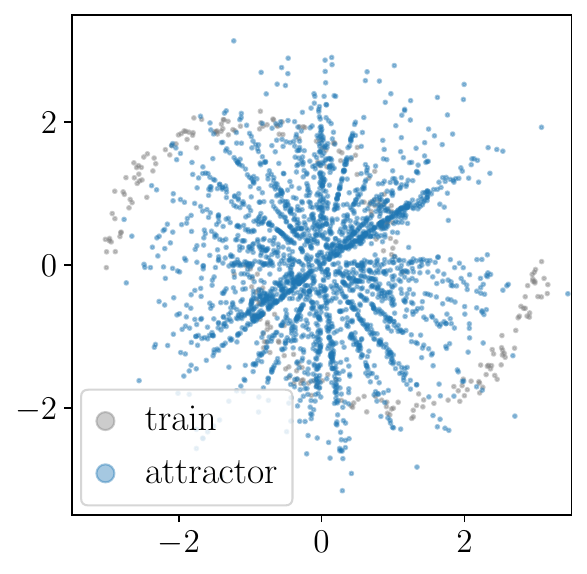}};
        \node at (0,0.0){\includegraphics[width=0.282\linewidth]{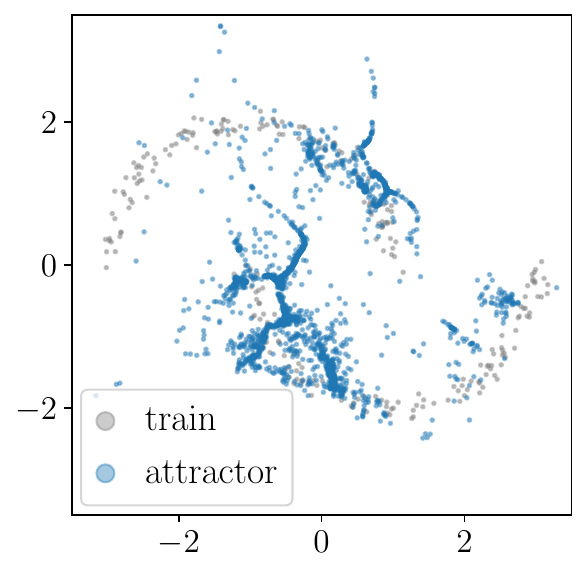}};
        \node at (5.5,0){\includegraphics[width=0.282\linewidth]{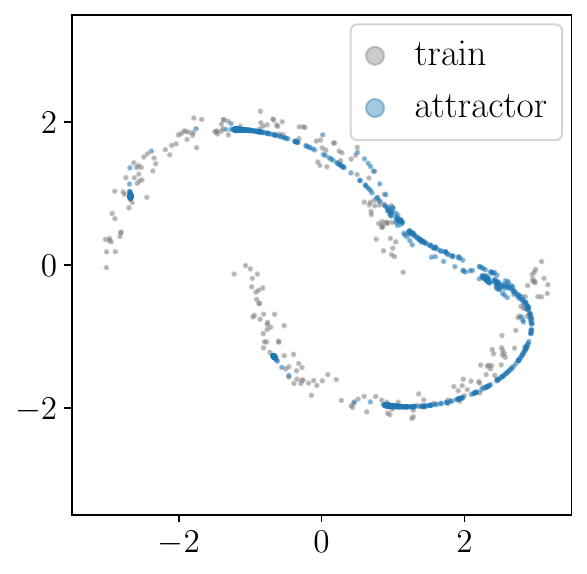}};
    \end{tikzpicture}    
    \vspace{-5mm}
    \caption{From left to right : attractors obtained with ResNet with zero bias, ResNet with learned bias and Transformer architectures trained with collage error minimization.}
    \label{fig:two-moons-resnet-tf}
\end{figure*}

\subsection{Consequences}

\paragraph{Convergence and Stability} Viewing ResNets, Transformers, and MoE layers as stochastic IFS enables the import of classical convergence and stability results: existence and uniqueness of invariant measures and characterization of attractors. For a P-IFS, the stochastic recursion $X_{t+1}=w_{\Xi_t}(X_t)$ and the deterministic update $x_{t+1}=F(x_t)$ with the conditional expected function $F(x)=\EE_{\Xi | X=x \sim \Cc(p(x))}[w_\Xi(x)]$ (the neural network) are in general distinct dynamical systems. Under the strong average Lipschitz condition of Theorem~\ref{thm:invariant-measure-pifs}, the function $F$ inherits contractivity and therefore has a unique global point attractor (see supplementary Section~\ref{app:avg-vs-expectation}). However, for linear/affine stochastic IFS, weaker Lyapunov-type ``average contractivity'' conditions (based on $\EE[\log c_\Xi]<0$)~\cite{furstenberg1960products,bougerol1985products} control only the stochastic recursion and do not guarantee stability of $F$. A simple one-dimensional example illustrating this distinction is provided in Appendix~\ref{app:avg-vs-expectation}. In the non-linear case, up to our knowledge there is no better hypotheses than the strong average Lipschitz condition. Lyapunov-type conditions can be imported to analyze gradients and to guarantee local stability of random dynamical systems~\cite{oseledets1968multiplicative,ruelle1979ergodic}.

\paragraph{Deep Learning in Practice} The IFS viewpoint brings interesting insight to the design and practice of deep learning. The IFS stability results require contractivity (either uniform or in the averaged sense of Theorem~\ref{thm:invariant-measure-pifs}), these conditions are too strong for deterministic networks and their depth dynamic would collapse to zero. Yet, some existing practices have come closer to this regime \eg{}, Lipschitz networks, spectral normalization~\cite{miyato2018spectral}, or OT entropy-regularized attention~\cite{sander2022sinkformers}. Sampling in deep neural networks is also a common practice as mentioned in Section~\ref{sec:tf-pifs} and noise injection~\cite{bishop1995training} is known to regularize. The main takeaway of the IFS viewpoint is that \emph{contractivity combined with sampling will guarantees stability while preserving expressivity}.

%The IFS viewpoint also clarifies which architectural practices are compatible with convergence guarantees.  
%Stability results require contractivity (either uniform or in the averaged sense of Theorem~\ref{thm:invariant-measure-pifs}), yet standard deep architectures explicitly violate this assumption: residual connections add the identity, batch normalization rescales activations in a data-dependent manner, and self-attention layers amplify directions selected by softmax weights. 
%As a consequence, most networks operate in a regime where neither contractivity nor existence of a global attractor can be expected.  The IFS framework therefore explains both (i) why deep models can exhibit unstable or chaotic dynamics in depth, and (ii) why convergence guarantees only emerge when architectures are explicitly constrained to be contractive 

%%%%%%%%%%%%%%
\section{Collage Error Bounds Generalization Error}
\label{sec:generalization}

\begin{defn}[Generalization, statistical and empirical collage error]
    \label{def:errors}
    Let $\mu$ be the data distribution on $\Xx$ and $\mu_n \eqdef \tfrac1n\sum_{i=1}^n \delta_{x_i}$ the empirical measure associated to samples $\{x_i\}_{i=1}^n \subset \Xx$. Let $T^\ast_\theta : \Pp(\Xx) \to \Pp(\Xx)$ a $c_\theta$-contractive operator on $(\Pp(\Xx), \mathrm{W}_2)$ parametrized by $\theta \in \Theta$ and let $\mu_\theta$ be its fixed point.
    
    The \emph{empirical collage error} is $\varepsilon(n,\theta) \eqdef \mathrm{W}_2\big( T^\ast_\theta{}\mu_n, \mu_n \big)$.
    
    The \emph{generalization error} is $\mathrm{W}_2(\mu_\theta,\mu)$.

    The \emph{statistical error} is $\mathrm{W}_2(\mu,\mu_n)$.

\end{defn}

\begin{thm}[Generalization bound]
    \label{thm:generalization-bound}
    Use the notation and hypotheses of Definition~\ref{def:errors}. Then, for any $(n,\theta)$ we have
    \eq{
        \mathrm{W}_2(\mu_\theta,\mu) \le \frac{\varepsilon(n,\theta)}{1-c_\theta}  + \mathrm{W}_2(\mu,\mu_n).
    }
\end{thm}

The first term $\varepsilon(n,\theta)$ is the empirical collage error, an objective that can be minimized in practice. The second term $\mathrm{W}_2(\mu,\mu_n)$ is purely statistical and depends only on sampling fluctuations of the empirical measure; under mild regularity assumptions on $\mu$ (finite $q$-th moment, bounded support, or subgaussian tails) it converges to $0$ at classical empirical OT rates~\cite{fournier2015rate}. Thus, whenever $T_\theta^\ast$ is contractive, minimizing $\varepsilon(n,\theta)$ controls the Wasserstein distance $\mathrm{W}_2(\mu_\theta,\mu)$ between the model's invariant law and the true data distribution. This bound can justify heuristically the use of the collage error as an evaluation metric.

\section{Numerical Experiments}

The IFS view on neural networks brings several stability/convergence results and a generalization bound. The contractivity condition is somehow strong and there is a tradeoff as the generalization bound explodes for $c_\theta$ close to $1$. In addition, the new collage error minimization objective might not be useful in practice for high-dimensional datasets. To answer these questions, we first design a simple MoE/I-IFS following Proposition~\ref{prop:moe-pifs}. We use spectral normalization of the linear layers used to define the IFS functions to ensure contractivity. Contrary to a standard Transformer we do use branch sampling during the forward pass ($+$ additional noise). We train this MoE/I-IFS on a simple 2D dataset supported on two separated half-circles plus Gaussian noise (two-moons datasets). Second, we compare the attractor of our MoE/I-IFS to the attractor of a small ResNet and a single-head Transformer (both standard non-stochastic and no imposed contractivity). We chose the hyper-parameters so that the total number of parameters is approximately $1300$ (see supplementary Table~\ref{tab:twomoons-hparams}). Lastly, we train bigger MoE/I-IFS on larger datasets (MNIST, CelebA, CIFAR-10) to show that minimizing the collage error is a valid and promising objective function for generative modeling. We refer to the supplementary Section~\ref{app:implementation} for implementation details.
\begin{figure*}%[b]
    \centering
    \begin{tikzpicture}
        \draw[step=1.0,white,thin] (-8.5,-2.5) grid (8.5,2.7); 
        \node at (-5.5,0){\includegraphics[width=0.3\linewidth]{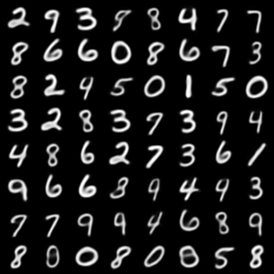}};
        \node at (0,0){\includegraphics[width=0.3\linewidth]{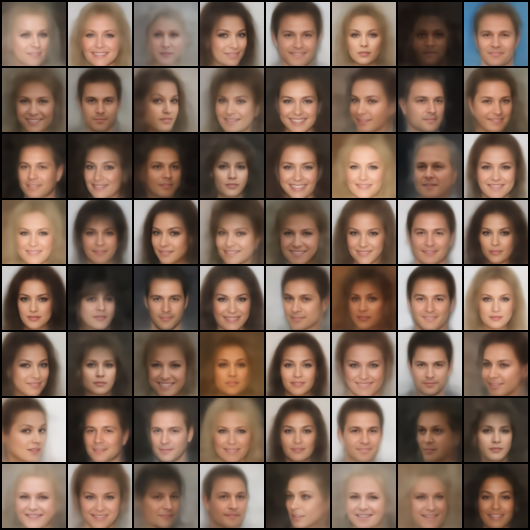}};
        \node at (5.5,0){\includegraphics[width=0.3\linewidth]{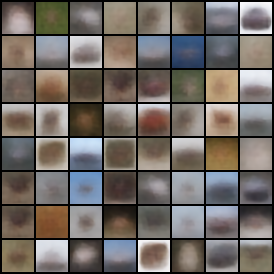}};
    \end{tikzpicture}    
    \caption{Samples from the learned MoE IFS on latent MNIST, CelebA, and CIFAR-10 representations (left to right).}
    \label{fig:latent-samples}
\end{figure*}

\subsection{Feasibility, Convergence and Error Bound}

We approximate the Wasserstein collage error using regularized OT (\texttt{geomloss}~\cite{feydy2019interpolating}) between samples of $T^\ast_\theta \mu_n$ and $\mu_n$. Figure~\ref{fig:two-moons-ifs} shows the successfully minimized collage error (left, $c_\theta=0.9$), the training set and the attractor of our trained MoE/I-IFS (center, $c_\theta=0.9$) and the generalization error and its bound for multiple contraction values $c_\theta \in \{0.3,0.5,0.7,0.9\}$. The attractor set matches the training set very well except where it is discontinuous. This is not a surprise as the functions of the IFS are continuous. Regarding the generalization bound, it is indeed tighter when $c_\theta$ is small enough. In low-dimension the statistical error is negligible, and, consistent with the theory, there is a scale factor proportional to $\nicefrac{1}{1-c_\theta}$ between the generalization error and its bound as $c_\theta$ increases. Finally, the $c_\theta$-error bound tradeoff is empirically mitigated (at least in our example). Indeed, the lower $c_\theta$ the tighter the bound, yet we observe that the generalization error increases and the bound decreases. In contrast, for larger $c_\theta$ the generalization error gets smaller while the bound worsens. In fact, the attractor is getting more constrained with lower contractivity constant (see supplementary Figure~\ref{supp-fig:two-moons-c-ifs}) but it remains close to the data distribution for moderate contractivity constants.

\subsection{Attractors}

ResNets are degenerate affine/linear P-IFS, their theoretical attractors are points/cones (when linear) and fractals (when affine). This is exactly what we get numerically in Figure~\ref{fig:two-moons-resnet-tf} (left-center). In both cases the attractors remain poorly aligned with the training set. Transformers are non-linear P-IFS, their attractors are more complex and they are in general more expressive. Hence, the attractor set matches partially the training set. Yet, we note that after 2000 epochs the objective collage error is still decreasing until 5000 epochs (supplementary Figure~\ref{supp-fig:tf_dim_10_10k_epochs}-left) and the displayed attractor (Figure~\ref{fig:two-moons-resnet-tf}-right) evolve with more training (supplementary Figure~\ref{supp-fig:tf_dim_10_10k_epochs}-right). In addition, we note that constraining the embedding dimension to the one of the dataset (\ie{} 2) prevent the transformer from learning non-linear structures, it collapses to a pair of Dirac (see supplementary Figure~\ref{supp-fig:tf_dim_2}-right). Such results are expected in theory~\cite{geshkovski2023mathematical}.

Together, the attractors of Figure~\ref{fig:two-moons-resnet-tf} must be compared to the MoE/I-IFS attractor presented in Figure~\ref{fig:two-moons-ifs}-center. To complete the comparison, we train a deterministic version of the MoE/I-IFS \ie{} a standard MoE without branch sampling (see Definition~\ref{def:moe-layer}). In this case the attractor is a Dirac at zero (even if we reactivate sampling to generate samples), see supplementary Figure~\ref{supp-fig:moe-det-attractors}.

\subsection{Larger Datasets}

Finally to test whether the collage objective can be useful to train generative models of real data, we train cascaded MoE/I-IFS on MNIST, CelebA and CIFAR-10. To this purpose, we first train an auto-encoder and then train the cascaded MoE on the latent distribution (see details in supplementary Section~\ref{app:implementation}). Finally, we can generate new samples by decoding power iterates of the IFS from any random latent initialization. The results displayed in Figure~\ref{fig:latent-samples} demonstrate that minimizing the collage objective leads to meaningful samples in higher dimension (32 to 256). Yet, the samples are worsening with the increasing complexity of the dataset (from MNIST to CIFAR-10). When trained on MNIST and Celeba, auto-encoder latent spaces distributions are known to have a manifold structure. For MNIST, the learned manifold captures variations in digit identity, stroke thickness, and slant. For CelebA, it organizes semantically meaningful factors such as pose, skin tone, and facial expression. This is less true for CIFAR-10 which is more diverse and has less systematic variability.

\section{Discussion and Conclusion}

We introduced a stochastic-IFS perspective on depth dynamics that makes the branch maps and selector kernels explicit for standard architectures, yielding existence/uniqueness of invariant measures under average contractivity and a Wasserstein generalization bound via the collage error. Empirically, the collage objective produces meaningful attractors for MoE/I-IFS and nontrivial samples in latent spaces, while highlighting the expressivity--stability tradeoff in high dimensions. More broadly, combining IFS branches with projection operators offers a way to study stability in other settings, such as DNN classifiers or AEs, where the dynamics act on constrained manifolds or feature subspaces.

Two limitations are intrinsic: (i) contractivity assumptions are strong, and deterministic networks without sampling typically collapse; (ii) the collage objective scales with OT estimation. Recent analyses of smooth attention and Lipschitz control in Transformers~\cite{castin2024smooth} suggest that data-dependent kernels can be regularized directly, and contraction arguments used in diffusion dynamics~\cite{gao2025wasserstein} point to a broader connection between generative flows and stochastic IFS. Future work should emphasize empirical extensions, including alternative discrepancies for the collage objective (sliced Wasserstein, MMD, energy distances) and training directly in data space without AE.

%\newpage

\bibliographystyle{icml2026/icml2026}
\bibliography{references}

\appendix

\section{Supplementary Definitions and Proofs}
\label{app:supp-def}

\begin{defn}[$c$-contraction]
    \label{def:c-contraction}
    Let $T : \mathcal X \to \mathcal X$. We say that $T$ is a $c$-contraction (or $c$-contractive) iff there exists a constant $c \in [0,1)$ such that for all $x,y \in \mathcal X$,
    \eq{\label{eq:banach-contraction}
        d(Tx,Ty) \le c d(x,y)
    }
\end{defn}

\begin{thm}[Banach Fixed Point Theorem]
    \label{thm:banach-fixed-point}
    Let $T:\Xx \to \Xx$ be $c$-contraction. Then:
    \begin{enumerate}
        \item There exists a unique fixed point $\tilde{x} \in \mathcal X$ such that
        $T \tilde{x} = \tilde{x}$.
        \item For any $x_0 \in \mathcal X$, the sequence defined by
        $x_{n+1} = T x_n$ converges to $\tilde{x}$.
        \item Moreover, for all $n \ge 0$,
        \begin{equation}
            d(x_n,\tilde{x})
            \;\le\;
            \frac{c^n}{1-c}\, d(x_1,x_0).
            \label{eq:banach-rate}
        \end{equation}
    \end{enumerate}
\end{thm}

\begin{proof}[Proof of Banach Fixed Point Theorem~\ref{thm:banach-fixed-point}]
Fix $x_0 \in \mathcal X$ and define the Picard iterates
$x_{n+1}=Tx_n$ for $n\ge 0$. Using \eqref{eq:banach-contraction},
\[
    d(x_{n+1},x_n)
    = d(Tx_n,Tx_{n-1})
    \le c\, d(x_n,x_{n-1}).
\]
By induction,
\[
    d(x_{n+1},x_n)
    \le c^n\, d(x_1,x_0), \qquad \forall n \ge 0.
\]
For $m>n$, the triangle inequality yields
\eqal{
    d(x_m,x_n)
    \le \sum_{k=n}^{m-1} d(x_{k+1},x_k)
    &\le \sum_{k=n}^{\infty} c^k\, d(x_1,x_0) \\
    & = \frac{c^n}{1-c}\, d(x_1,x_0).
}
Hence $(x_n)_{n\ge 0}$ is a Cauchy sequence. Since $(\mathcal X,d)$ is
complete, there exists $\tilde{x} \in \mathcal X$ such that $x_n \to \tilde{x}$.

We now show that $\tilde{x}$ is a fixed point. By continuity of $T$
(which follows from \eqref{eq:banach-contraction}),
\[
    T\tilde{x} = T\bigl(\lim_{n\to\infty} x_n\bigr)
    = \lim_{n\to\infty} T x_n
    = \lim_{n\to\infty} x_{n+1}
    = \tilde{x}.
\]
Thus $\tilde{x}$ is a fixed point.

To prove uniqueness, suppose $y^\ast$ is another fixed point,
$Ty^\ast = y^\ast$. Then
\[
    d(\tilde{x},y^\ast)
    = d(T\tilde{x},Ty^\ast)
    \le c\, d(\tilde{x},y^\ast).
\]
Since $c<1$, this implies $(1-c)\,d(\tilde{x},y^\ast)\le 0$, hence
$d(\tilde{x},y^\ast)=0$ and $\tilde{x}=y^\ast$.

Finally, the estimate \eqref{eq:banach-rate} follows by letting $m\to\infty$
in the previous bound:
\[
    d(x_n,\tilde{x})
    \le \frac{c^n}{1-c}\, d(x_1,x_0).
\]
This completes the proof.
\end{proof}

\begin{proof}[Proof of Banach Fixed Point Corollary~\ref{cor:metric-collage}]
Using $T\tilde{x} = \tilde{x}$ and the triangle inequality,
\[
    d(\tilde{x},y)
    = d(T\tilde{x},y)
    \le d(T\tilde{x},Ty) + d(Ty,y).
\]
By the contraction property \eqref{eq:banach-contraction},
\[
    d(T\tilde{x},Ty) \;\le\; c\, d(\tilde{x},y),
\]
and by definition of $\varepsilon$ we have $d(Ty,y) = \varepsilon$.
Thus
\[
    d(\tilde{x},y)
    \;\le\;
    c\, d(\tilde{x},y) + \varepsilon.
\]
Rearranging gives
\[
    (1-c)\, d(\tilde{x},y) \;\le\; \varepsilon,
\]
which is equivalent to \eqref{eq:collage-ineq}.
\end{proof}

\begin{proof}[Proof of Proposition~\ref{prop:support-invariant-measure-hutchinson}]
    Let $K_{\tilde{\mu}} \eqdef \operatorname{supp}(\tilde{\mu})$. We have
    \begin{align*}
       K_{\tilde{\mu}}
         &= \operatorname{supp}(T^\ast \tilde{\mu}) \\
         &= \operatorname{supp}\left(\sum_{\xi\in\Ii} (w_\xi)\#(q_\xi\cdot\tilde{\mu})\right) \\
         &= \overline{\bigcup_{\xi\in\Ii} w_\xi(\operatorname{supp}(q_\xi\cdot\tilde{\mu}))} \\
         &= \overline{\bigcup_{\xi\in\Ii} w_\xi(\operatorname{supp}(\tilde{\mu}))}.
    \end{align*}
    The equality from the second to the third line holds because the support of
    the sum is the closure of the union of the individual supports. The last
    equality holds because for all $\xi \in \Ii$, $q_\xi>0$. Closed-ness of
    support and continuity for the functions gives
    \[
        K_{\tilde{\mu}} = \Hh(K_{\tilde{\mu}}),
    \]
    \ie{} $K_{\tilde{\mu}}$ is a fixed point of the Hutchinson operator on
    compacts. Contractivity of the $w_\xi$ makes that operator a contraction on
    $(\mathcal{K}(\Xx),d_H)$ with unique fixed point $\tilde{K}$. Hence
    $K_{\tilde{\mu}} = \tilde{K}$.
\end{proof}

\begin{proof}[Proof of Proposition~\ref{prop:residual-ifs}]
    We verify easily that the Markov recursion is conform to Definition~\ref{def:place-ifs}.
\end{proof}

\begin{proof}[Proof of Proposition~\ref{prop:transformerIFS}]
    Conditioned on $\bx$, the indices $(\Xi_i)$ are independent with marginals
    $(\alpha_{i,\cdot}(\bx))$. Sampling
    \eq{
        \Xi \mid \bX=\bx \sim \Cc(p_{\mathrm{TF}}(\bx))
    }
    and applying $w^\theta_\Xi$ yields a well-defined Markov recursion and
    therefore a place-dependent IFS on $\Xx^n$. Taking conditional expectation
    over $\Xi$ produces $G_\theta$ as claimed.
\end{proof}

\begin{proof}[Proof of Proposition~\ref{prop:moe-pifs}]
    We verify easily that the Markov recursion is conform to Definition~\ref{def:place-ifs}.
\end{proof}

\begin{proof}[Proof of Theorem~\ref{thm:invariant-measure-pifs}]
    See~\citet{stenflo2002uniqueness} Section 4 and references therein.
\end{proof}

\begin{proof}[Proof of Theorem~\ref{thm:stochastic-collage}]
Let $(\Ww,q)$ be a $c$-contractive I-IFS, so each $w_\xi$ is $c$-Lipschitz.
For any $\mu,\nu\in\mathcal{P}(\Xx)$,
\eqal{
    \mathrm{W}_2^2(T^\ast\mu,T^\ast\nu)
    &= \mathrm{W}_2^2\Big(\sum_{\xi\in\Ii} q_\xi (w_\xi)_\#\mu,\ 
        \sum_{\xi\in\Ii} q_\xi (w_\xi)_\#\nu\Big) \\
    &\le \sum_{\xi\in\Ii} q_\xi\,
        \mathrm{W}_2^2\big((w_\xi)_\#\mu,(w_\xi)_\#\nu\big) \\
    &\le \sum_{\xi\in\Ii} q_\xi\, c^2\,\mathrm{W}_2^2(\mu,\nu)
      = c^2\,\mathrm{W}_2^2(\mu,\nu).
}
The first inequality uses convexity of $\mathrm{W}_2^2$ and the second uses the
pushforward stability $\mathrm{W}_2((w_\xi)_\#\mu,(w_\xi)_\#\nu)\le
\mathrm{Lip}(w_\xi)\,\mathrm{W}_2(\mu,\nu)$; see \citet{villani2008optimal}.
Thus $T^\ast$ is $c$-contractive in $\mathrm{W}_2$, and Corollary~\ref{cor:metric-collage}
applies.
\end{proof}

\begin{proof}[Proof of Theorem~\ref{thm:generalization-bound}]
    Let $\mu_\theta$ be the fixed point of $T^\ast_\theta$. By the triangle
    inequality,
    \eq{
        \mathrm{W}_2(\mu_\theta,\mu)
          \le \mathrm{W}_2(\mu_\theta,\mu_n) + \mathrm{W}_2(\mu_n,\mu).
    }
    Since $T^\ast_\theta$ is $c_\theta$-contractive and
    $\varepsilon(n,\theta)=\mathrm{W}_2(T^\ast_\theta\mu_n,\mu_n)$, we have
    \eqal{
        \mathrm{W}_2(\mu_\theta,\mu_n)
        &\le \mathrm{W}_2(\mu_\theta,T^\ast_\theta\mu_n)
          + \mathrm{W}_2(T^\ast_\theta\mu_n,\mu_n) \\
        &\le c_\theta\,\mathrm{W}_2(\mu_\theta,\mu_n) + \varepsilon(n,\theta).
    }
    Rearranging yields
    \[
        \mathrm{W}_2(\mu_\theta,\mu_n)
          \le \frac{\varepsilon(n,\theta)}{1-c_\theta},
    \]
    which gives the claimed bound.
\end{proof}

%%%%%%%%%%%%%
\begin{table}[H]
  \centering
  \small
  \setlength{\tabcolsep}{2pt}  % tighten columns
  \renewcommand{\arraystretch}{1.2}
  \begin{tabular}{lcccc}
  \toprule
  \textbf{Model} & \textbf{Maps $w_i$} & \textbf{Kernel $p_\xi(x)$} &
  \textbf{IFS type} \\
  \midrule
  \makecell[l]{ReLU residual \\ block  (sampled)} &
    \makecell[l]{Definition \ref{def:residual-branch}} &
    $\delta_{\xi(x)}$ &
    \makecell[c]{P-IFS \\ (degenerate)} \\
  \makecell[l]{Transformer \\ block (sampled)} &
    \makecell[l]{Definition \ref{def:transform-branch}} &
    $\prod_i \alpha_{i,\xi_i}(\bx)$ &
    P-IFS \\
  MoE (sampled) &
    Generic &
    Generic &
    P-IFS \\
  \bottomrule
\end{tabular}
\caption{Summary of IFS interpretations across architectures.}
\label{tab:ifs-architectures}
\end{table}

\section{Softplus ResNet as an I-IFS}
\label{app:softplus-ifs}

Construction of IFSs associated to neural networks is not only feasible for ReLU activations. It is in fact always feasible. Below we make a construction for a ResNet with Softplus activation.

\begin{defn}[Softplus residual block]
Let $\phi(z)=\log(1+e^z)$ be the softplus activation applied coordinatewise. Given $A\in\RR^{m\times d}$, $b\in\RR^m$, $B\in\RR^{d\times m}$, $c\in\RR^d$, define the residual block $F_\theta:\Xx\to\Xx$ by
\[
    F_\theta(x)=x + B\,\phi(Ax+b)+c,
    \qquad \theta=(A,B,b,c).
\]
\end{defn}

\begin{defn}[Softplus branch family and kernel]
Let $\rho$ be the logistic density. Let $\Ii=\RR^m$ and for $\tau\in\Ii$ define the branch map
\[
    w_\tau(x)=x + B\,\mathrm{ReLU}\bigl(Ax+b-\tau\bigr) + c,
\]
where $\mathrm{ReLU}$ is applied coordinatewise. Let $p$ be the product measure with density $\prod_{j=1}^m \rho(\tau_j)$ on $\Ii$.
\end{defn}

\begin{prop}[Softplus residual block as I-IFS]
The pair $(\Ww_{\mathrm{sp}},p)$ with $\Ww_{\mathrm{sp}}=\{w_\tau\}_{\tau\in\Ii}$ defines a place-independent IFS on $\Xx$, and the softplus residual block satisfies
\[
    F_\theta(x)=\EE_{\tau\sim p}[w_\tau(x)].
\]
\end{prop}

\begin{proof}
For a scalar $z$, $\phi(z)=\EE_{T\sim\rho}\big[\mathrm{ReLU}(z-T)\big]$ where $\rho$ is the logistic density. Applying this coordinatewise and using independence of the components of $\tau$ yields the stated expectation identity.
\end{proof}

\section{Stability of IFS \vs{} Expected IFS ?}
\label{app:avg-vs-expectation}

We distinguish two notions of average contractivity for a place-dependent P-IFS $(\{w_\xi\}_{\xi\in\Ii},p)$ on $(\Xx,d)$ with selector kernel $(p_\xi(x))_{\xi\in\Ii}$ and Lipschitz constants $c_\xi$ for each $w_\xi$.

\begin{defn}[Strong average Lipschitz contractivity]
We say that the P-IFS is strongly average-contractive if
\begin{equation}
  \label{eq:strong-average-Lip}
  \sup_{x\in\Xx}
  \sum_{\xi\in\Ii} p_\xi(x)\,c_\xi
  \;\le\; c < 1,
\end{equation}
and the selector kernel $p: \Xx \to \Delta^{|\Ii|-1}$ depends on $x$ in a Lipschitz way (\eg{}\ in total variation or Wasserstein distance), as in Theorem~\ref{thm:invariant-measure-pifs}.
\end{defn}

\begin{defn}[Lyapunov (log-average) contractivity]
We say that the P-IFS is Lyapunov-average-contractive if, for the random Lipschitz constants $c_{\Xi_t}$ along a trajectory, one has
\begin{equation}
  \label{eq:lyap-average}
  \limsup_{t\to\infty}\frac{1}{t}
  \sum_{k=0}^{t-1}\log c_{\Xi_k} \;<\; 0,
  \qquad \text{a.s.}
\end{equation}
In the place-independent case, this reduces to $\EE[\log c_{\Xi}]<0$. This is the condition classically used in the theory of products of random matrices.
\end{defn}

In our framework the deterministic map associated with the P-IFS is the barycentric map
\begin{equation}
  F(x) \;=\; \EE[w_\Xi(x)\mid X=x]
          \;=\; \sum_{\xi\in\Ii} p_\xi(x)\,w_\xi(x).
\end{equation}
The following proposition clarifies when $F$ inherits contractivity and when it does not.

\begin{prop}[Inheritance under strong average Lipschitz]
  \label{prop:inherit-strong-average}
  Assume the strong average Lipschitz condition \eqref{eq:strong-average-Lip} and that each $w_\xi$ is $c_\xi$-Lipschitz on $(\Xx,d)$. Then the barycentric map $F$ is a contraction with Lipschitz constant at most $c<1$, \ie{}
  \[
    d\big(F(x),F(y)\big)
    \;\le\;
    c\,d(x,y)
    \qquad\forall x,y\in\Xx.
  \]
  In particular, $F$ has a unique fixed point and the deterministic recursion $x_{t+1}=F(x_t)$ converges to it for any initial condition.
\end{prop}

\begin{proof}
For any $x,y\in\Xx$,
\begin{align*}
  d\big(F(x),F(y)\big)
    &= d\Big(\sum_{\xi} p_\xi(x)\,w_\xi(x),
             \sum_{\xi} p_\xi(y)\,w_\xi(y)\Big) \\
    &\le
      d\Big(\sum_{\xi} p_\xi(x)\,w_\xi(x),
             \sum_{\xi} p_\xi(x)\,w_\xi(y)\Big)\\
    & \quad + d\Big(\sum_{\xi} p_\xi(x)\,w_\xi(y),
              \sum_{\xi} p_\xi(y)\,w_\xi(y)\Big).
\end{align*}
Under the regularity assumptions on $p(\cdot\mid x)$, the second term can be bounded by a multiple of $d(x,y)$ (as in Theorem~\ref{thm:invariant-measure-pifs}), while the first term satisfies
\eqal{
  d\Big(&\sum_{\xi} p_\xi(x)\,w_\xi(x),
        \sum_{\xi} p_\xi(x)\,w_\xi(y)\Big) \\
  &\le
  \sum_{\xi} p_\xi(x)\,d\big(w_\xi(x),w_\xi(y)\big) \\
  &\le 
  \Big(\sum_{\xi} p_\xi(x)c_\xi\Big)\,d(x,y).
}
Taking the supremum over $x$ and using \eqref{eq:strong-average-Lip} yields the claimed contraction bound for $F$ (up to adjusting constants to absorb the kernel regularity term).
\end{proof}

\begin{exmp}[Failure under Lyapunov-average contractivity]
\label{ex:avg-vs-expectation}
Consider $\Xx=\RR$ with two maps
\[
  w_1(x) = \tfrac12 x,
  \qquad
  w_2(x) = 2x,
\]
and a place-independent selector kernel $\PP(\Xi=1)=p$, $\PP(\Xi=2)=1-p$. Each $w_i$ is linear with Lipschitz constants $c_1=\tfrac12$, $c_2=2$, so
\[
  \EE[\log c_{\Xi}]
    = p\log\tfrac12 + (1-p)\log 2
    = (1-2p)\log 2.
\]
For $p\in(\tfrac12,1)$ this Lyapunov exponent is negative and the random recursion $X_{t+1}=w_{\Xi_t}(X_t)$ is contracting on average in the sense of
\eqref{eq:lyap-average}.  However the barycentric map is
\[
  F(x)
    = \EE[w_\Xi(x)]
    = p\,\tfrac12 x + (1-p)\,2x
    = (2-\tfrac{3p}{2})x,
\]
with Lipschitz constant $|2-\tfrac{3p}{2}|>1$ whenever $p\in(\tfrac12,\tfrac23)$. In this regime the random IFS is Lyapunov-average-contractive, but the deterministic map $F$ is expanding and its iterates diverge for all $x\neq 0$.  Thus Lyapunov-average contractivity of the random P-IFS does not imply contractivity or stability of its deterministic barycentric map.
\end{exmp}

This example shows that our strong average Lipschitz condition \eqref{eq:strong-average-Lip} is precisely the regime in which stability and convergence are inherited by the deterministic update, whereas weaker Lyapunov-type conditions only control the stochastic recursion and have no direct implication for $F$.

\section{Deep MoE}
\label{app:deep-moe}

\begin{defn}[Deep Mixture-of-Experts as a P-IFS]
    \label{def:deep-moe}
    Let each layer $d\in\{1,\dots,D\}$ consist of a selector kernel
    $p_{\mathrm{MoE}}^{(d)}= p^{(d)}: \Xx \to \Delta^{K-1}$ on a finite index set $\Ii^{(d)}$ and branch maps
    $\{ w^{(d)}_\xi \}_{\xi\in\Ii^{(d)}}$.  
    A depth-$D$ MoE defines a place-dependent IFS on $\Xx$ through the random
    cascade
    \eqal{
        x_{t+1}
          = w^{(D)}_{\Xi^{(D)}_t} & \circ \cdots \circ w^{(1)}_{\Xi^{(1)}_t}(x_t)\\
                                  & \Xi^{(d)}_t | X=x_t \sim \Cc(p^{(d)}(x_t)).
    }
    The deterministic update is the conditional expectation
    \eq{
        H(x)
          = \EE\!\left[
                w^{(D)}_{\Xi^{(D)}} \circ \cdots \circ
                w^{(1)}_{\Xi^{(1)}}(x)
              \right],
    }
    and the associated transfer operator is the composition
    \eq{
        T^\ast
          = T_D^\ast \circ \cdots \circ T_1^\ast,
        \qquad
        T_d^\ast \mu
          = \sum_{\xi\in\Ii^{(d)}}
                (w^{(d)}_\xi) \# \bigl( p^{(d)}_\xi \cdot \mu \bigr).
    }
\end{defn}

This amounts to consider depth-dependent recursion and it increases the number of possible maps. Overall, it can be seen as a one-layer IFS with a multi-index set to account for all possible branches.

\section{Implementation Details}
\label{app:implementation}

\paragraph{Two‑moons.} We generate 2D point clouds with 2,048 training points and a 50k point reference set ($\mathrm{noise}=0.1$, $\mathrm{radius}=2.0$). Models follow Table~\ref{tab:twomoons-hparams}(MoE/ResNet/Transformer) and are trained with the stochastic collage objective using Sinkhorn loss ($\mathrm{blur}=0.05$), batch size 256, 2k epochs, Adam ($\mathrm{lr}=1e‑3$). For MoE runs we use contraction ($c<1$) and ($\sigma=0.04$); ResNet/Transformer are non‑contractive unless spectral normalization is enabled. Dynamics/attractor plots use 50 Markov steps; when computing the bound we use $\mathrm{burn‑in}=100$ and Monte‑Carlo  estimates with 8 batches of size 512.

\begin{table}[H]
\centering
\small
\setlength{\tabcolsep}{2pt}  % tighter columns
\begin{tabular}{lccc}
    \toprule
    & MoE IFS & ResNet & Transformer \\
    \midrule
    Base map &
    \makecell[l]{Nonlinear \\ MLP branches} &
    \makecell[l]{Repeated \\ residual block} &
    \makecell[l]{Self-attention \\ on contexts} \\
    Width &
    32 &
    32 &
    10 \\
    Depth &
    1 &
    8 &
    1 \\
    Routing &
    $K = 8$ & --- & $1$ head \\
    Contraction $c$ &
    $c < 1$ &
    None & None \\
    Noise $\sigma$ &
    $0.04$ & $0$ & $0$ \\
    Context size &
    --- & --- & $n = 2$ \\
    \# params &
    1304 & 1296 & 1382 \\
    \bottomrule
\end{tabular}
\caption{Key hyperparameters for the two-moons experiments. The width of the transformers is the dimension of the input embedding.}
\label{tab:twomoons-hparams}
\end{table}

\paragraph{MNIST/CIFAR‑10/CelebA.} We train in latent space using a ConvResNetAE (MSE reconstruction) and then fit a latent IFS on encoded samples. Datasets are resized to $32\times 32$ for MNIST/ CIFAR‑10 and center‑cropped ($148$) then resized to $64\times 64$ for CelebA; we use latent dims/base channels of 32/64 (MNIST) and 256/128 (CIFAR‑10, CelebA). The latent IFS is a stacked NonlinearContractiveIFS (see Definition~\ref{def:deep-moe}) with $D=6$ stages, $K=128$ maps, branch depth 16 (depth of individual maps $w_i$), hidden dim 256, contraction ($c=0.99$) (warm‑up from ($c_{\text{init}}=10$) over 12 epochs), and ($\sigma=0.01$); routing uses constant‑IFS with learned logits ($\mathrm{temperature}=1.0$). We fit on 40k latents with Sinkhorn loss in latent space using Adam ($\mathrm{lr}=1e‑5$), batch size 512, 30 epochs with ReduceLROnPlateau, and generate samples by iterating $50$ IFS steps and decoding $64$ latents.

\section{Supplementary Figures}

\begin{figure*}%[H]
    \centering
    \begin{tikzpicture}
        \draw[step=1.0,white,thin] (-8.5,-2.5) grid (8.5,2.7); 
        \node at (-5.5,0){\includegraphics[width=0.3\linewidth]{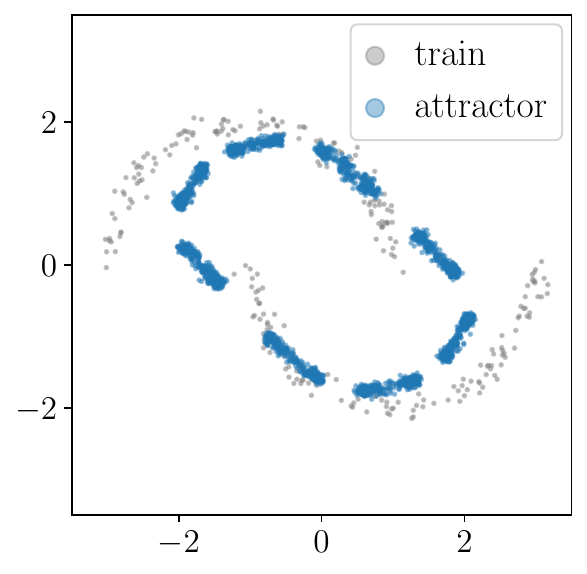}};
        \node at (0,0.0){\includegraphics[width=0.3\linewidth]{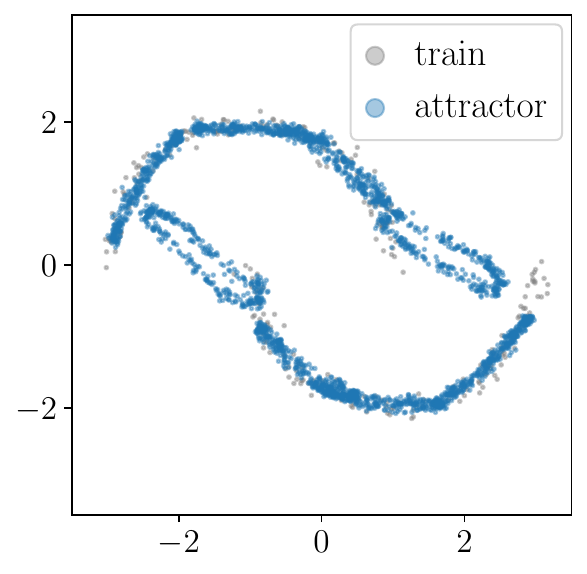}};
        \node at (5.5,0){\includegraphics[width=0.3\linewidth]{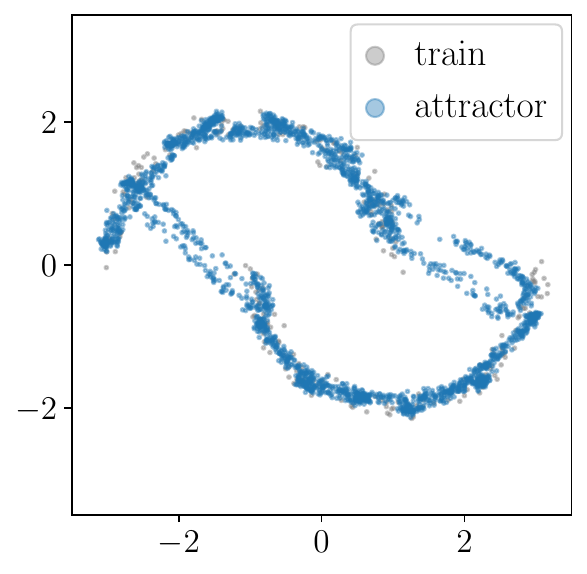}};
    \end{tikzpicture}    
    \vspace{-5mm}
    \caption{Training data and samples of the attractor measure of the trained IFS for contraction constants $c \in \{0.3,0.5,0.7\}$ (left to right).}
    \label{supp-fig:two-moons-c-ifs}
\end{figure*}

\begin{figure*}%[H]
    \centering
    \begin{tikzpicture}
        \draw[step=1.0,white,thin] (-8.5,-2.5) grid (8.5,2.7); 
        \node at (-3.5,0){\includegraphics[width=0.3\linewidth]{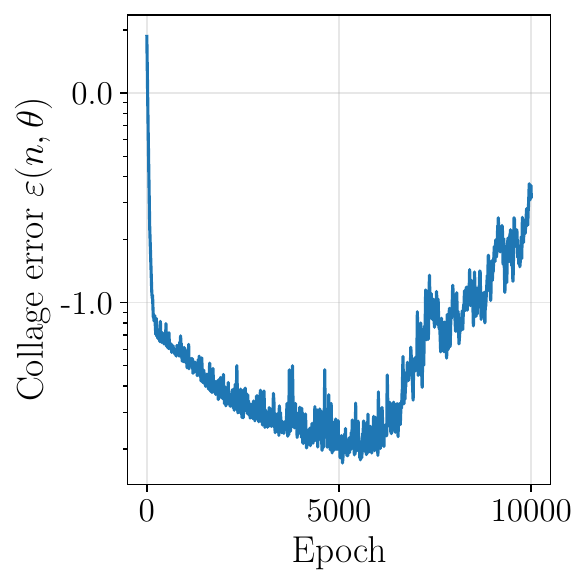}};
        \node at (3.5,0.15){\includegraphics[width=0.282\linewidth]{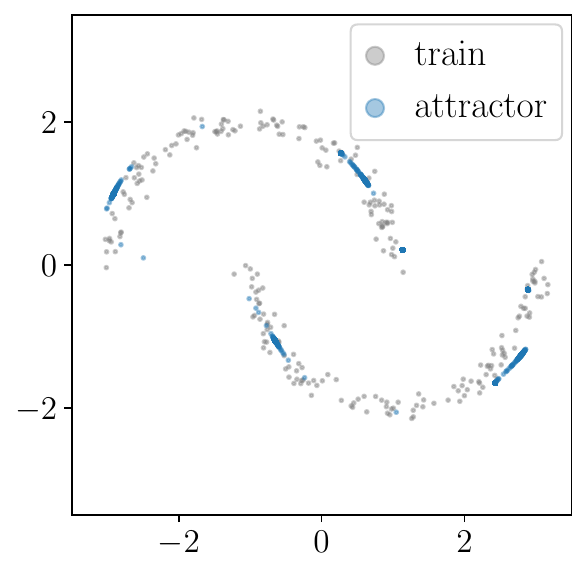}};
    \end{tikzpicture}    
    \vspace{-5mm}
    \caption{Empirical collage error $\varepsilon(n,\theta)$ over 10k epoch for transformer with embedding dimension $10$ (left) and its attractor (right).}
    \label{supp-fig:tf_dim_10_10k_epochs}
\end{figure*}

\begin{figure*}%[H]
    \centering
    \begin{tikzpicture}
        \draw[step=1.0,white,thin] (-8.5,-2.5) grid (8.5,2.7); 
        %\node at (-5.5,0){};
        \node at (-3.5,0.0){\includegraphics[width=0.3\linewidth]{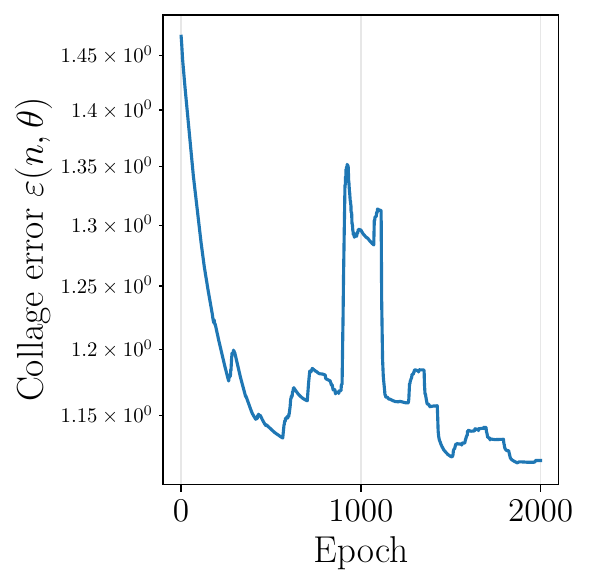}};
        \node at (3.5,0.15){\includegraphics[width=0.282\linewidth]{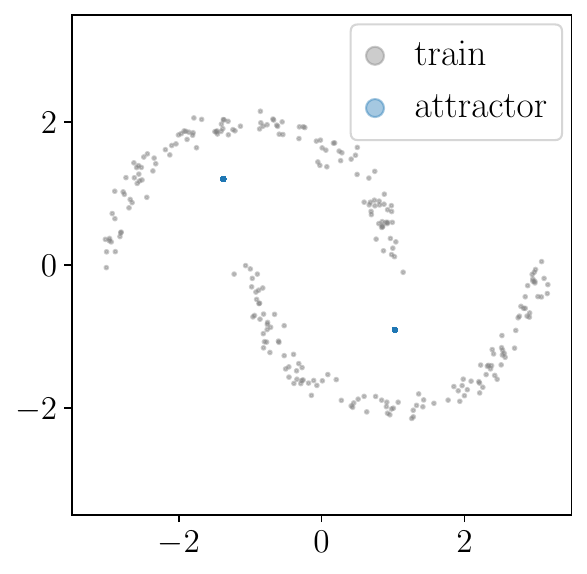}};
    \end{tikzpicture}    
    \vspace{-5mm}
    \caption{Empirical collage error $\varepsilon(n,\theta)$ for transformer with embedding dimension $2$ (left) and its attractor (right).}
    \label{supp-fig:tf_dim_2}
\end{figure*}

\begin{figure*}%[H]
    \centering
    \begin{tikzpicture}
        \draw[step=1.0,white,thin] (-8.5,-2.5) grid (8.5,2.7); 
        %\node at (-5.5,0){};
        \node at (-3.5,0.0){\includegraphics[width=0.3\linewidth]{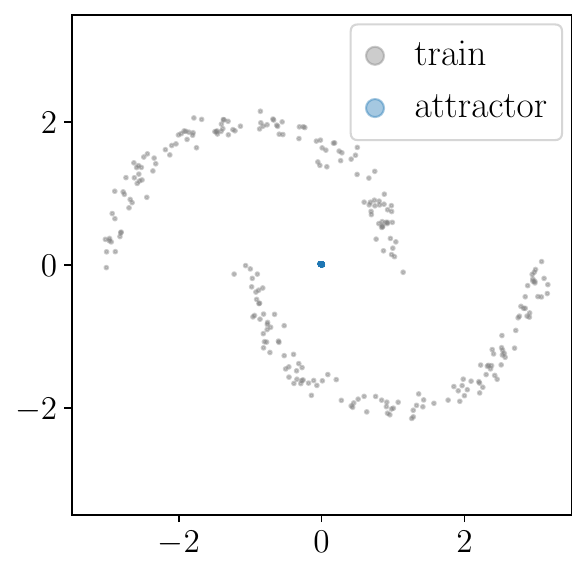}};
        \node at (3.5,0.0){\includegraphics[width=0.3\linewidth]{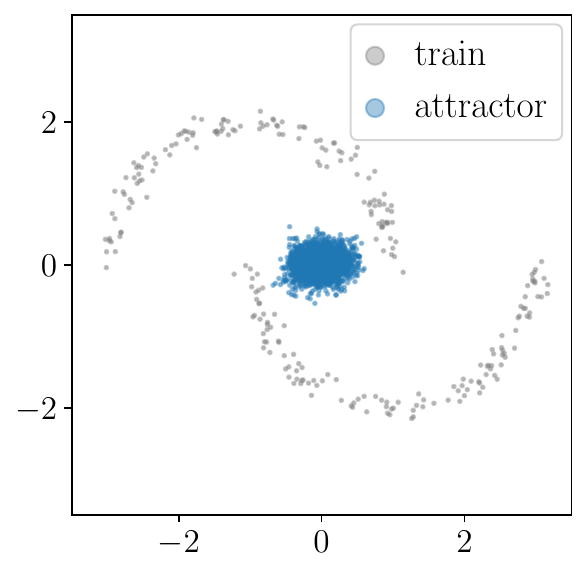}};
    \end{tikzpicture}    
    \vspace{-5mm}
    \caption{Attractors of deterministic MoE/I-IFS. Left: determinisctic sampling. Right: stochastic sampling.}
    \label{supp-fig:moe-det-attractors}
\end{figure*}

\end{document}